\documentclass[twoside]{article}

\usepackage{hyperref}
\usepackage{amsmath,amssymb,amsthm,amsfonts,latexsym}
\usepackage{url}
\usepackage{graphicx}
\usepackage{bm}
\usepackage{subfigure}
\usepackage{color}
\usepackage{appendix}
\usepackage{amsthm}
\usepackage{enumitem}

\usepackage{algorithm,algpseudocode,float}

\usepackage[mathcal]{eucal}
\usepackage[accepted]{aistats2019}

\newtheorem{theorem}{Theorem}

\newtheorem{proposition}{Proposition}

\theoremstyle{remark}
\newtheorem{rem}{Remark}[section]

\begin{document}
\twocolumn[
\aistatstitle{A discrete version of CMA-ES}
\aistatsauthor{Eric Benhamou$^{\dagger}$, $^{\ddagger}$, Jamal Atif$^{\ddagger}$, Rida Laraki$^{\ddagger}$}
\aistatsaddress{eric.benhamou@aisquareconnect.com or eric.benhamou@dauphine.eu, \\ jamal.atif@dauphine.fr, rida.laraki@dauphine.fr} 
\aistatsaddress{A.I. Square Connect$^{\dagger}$ \hspace{1cm} Lamsade, Universite Paris Dauphine, PSL$^{\ddagger}$} 
]

\begin{abstract}
Modern machine learning uses more and more advanced optimization techniques to find optimal hyper parameters. Whenever the objective function is non-convex, non continuous and with potentially multiple local minima, standard gradient descent optimization methods fail. A last resource and very different method is to assume that the optimum(s), not necessarily unique, is/are distributed according to a distribution and iteratively to adapt the distribution according to tested points. These strategies originated in the early 1960s, named Evolution Strategy (ES) have culminated with the CMA-ES (Covariance Matrix Adaptation) ES. It relies on a multi variate normal distribution and is supposed to be state of the art for general optimization program. However, it is far from being optimal for discrete variables.
In this paper, we extend the method to multivariate binomial correlated distributions. For such a distribution, we show that it shares similar features to the multi variate normal: independence and correlation is equivalent and correlation is efficiently modeled by interaction between different variables. We discuss this distribution in the framework of the exponential family. We prove that the model can estimate not only pairwise interactions among the two variables but also is capable of modeling higher order interactions. 
This allows creating a version of CMA ES that can accomodate efficiently discrete variables. We provide the corresponding algorithm and conclude.
\end{abstract}

\section{Introduction}
When facing an optimization problem where there is no access to the objective function's gradient or the objective function's gradient is not very smooth, the state of the art techniques rely on stochastic and derivative free algorithm that change radically the point of view of the optimization program. Instead of a deterministic gradient descent, we take a Bayesian point of view and assumes that the optimum is distributed according to a prior statistical distribution and uses particles or random draws to gradually update our statistical distribution. 
Among these method, the covariance matrix adaptation evolution strategy (CMA-ES; e.g., \cite{Hansen_2001, Hansen_2003}) has emerged as the leading stochastic and derivative-free algorithm for solving continuous optimization problems, i.e., for finding the optimum denoted by $\mathbf{x}^*$ of a real-valued objective function $f$, defined on a subset of a multi dimensional space of dimension $d$: $\mathbb{R}^d$.This method generates candidate points $\{\mathbf{x}_{i}\}$, $i \in \{1, 2, \dots, \lambda\}$, from a multivariate Gaussian distribution. It evaluates their objective function (also called fitness) values $\{f(\mathbf{x}_{i})\}$. As the distribution is characterized by its two first moments, it updates the mean vector and covariance matrix by using the sampled points and their fitness values, $\{(\mathbf{x}_{i}, f(\mathbf{x}_{i}))\}$. The algorithm keeps repeating the sampling-evaluation-update procedure (which can be seen like an exploration exploitation method until the distribution contracts to a single point or reaches the maximum of iterations. Convergence is measured either by a very small covariance matrix. The different variations around the original method  to improve the convergence investigate various heuristic method to update the distribution parameters. This strongly determines the behavior and efficiency of the whole algorithm. The theoretical foundation of the CMA-ES are that for continuous variables with given first two moments, the maximum entropy distribution is the normal distribution and that the update is based on a maximum-likelihood estimation, making this method based on a statistical principle.

A natural question is to adapt this method for discrete variables. Surprisingly, this has not been done before as the Gaussian distribution is a continuous time distribution inappropriate to discrete variables. One needs to change the underlying distirbution and also find the way to correlate the marginal distribution which can be tricky. However, we show in this paper that multivariate binomials are the natural discrete counterpart of Gaussian distributions. Hence we are able to change CMA Es to accommodate for discrete variables. This is the subject of this paper. In the section \ref{prim:sec}, we introduce the multivariate binomial distribution. Presenting this distribution in the general setting of exponential family, we can easily derive various properties and connect this distribution to maximum entropy. We also proved that for the assumed correlation structure, independence and correlation are equivalent, which is a also a feature of Gaussian distributions. In section \ref{sec:algorithm}, we present the algorithm. 

\section{Primer on Multivariate Binomials}\label{prim:sec}
\subsection{Intuition}
To start building some intuition on multivariate binomials, we start by the simplest case, that is a two dimensional Bernoulli. It is the extension to two dimensions of the univariate Bernoulli distribution. A Bernoulli random variable $X$, is a discrete variable that takes the value $1$ with probability $p$ and $0$ otherwise. The usual notation for the probability mass function is

$$
\mathbb{P}(X=x) = p^x (1-p)^{1-x} , x \in\{0,1\}
$$

A natural extension is to consider the random vector $X = (X_1, X_2)$. It takes values in the
Cartesian product space $\{0, 1\}^2 = \{0, 1\} \times\{0, 1\}$. If we denote the joint probabilities
$p_{ij} = P(X_1 = i, X_2 = j)$ for $i,j \in \{0, 1\}$, then the probability for the bivariate Bernoulli writes:
\begin{eqnarray}\label{bivariate}
\mathbb{P}(X \! = \! x)\! & \hspace{-0.3cm} =& \hspace{-0.3cm}  \mathbb{P}(X_1=x_1, X_2=x_2) \nonumber \\
\! &\hspace{-0.2cm}  =& \hspace{-0.3cm}  p_{11}^{x_1x_2}p_{10}^{x_1(1-x_2)}p_{01}^{(1-x_1)x_2}p_{00}^{(1-x_1)(1-x_2)} 
\end{eqnarray}
with the side conditions that the joint probabilities are between $0$ and $1$: for $i,j \in \{0, 1\}$, $0 \leq p_{ij} \leq 1$ 
and they sum to one $p_{00} + p_{10} + p_{01} + p_{11} = 1$

It is however better to write the joint distribution in terms of canonical parameters of the related exponential family. 
Hence, if we define
\begin{eqnarray}
\theta_1  &=& \log \biggl(\frac{p_{10}}{p_{00}} \biggr), \label{theta1} \\
\theta_2 &=& \log \biggl(\frac{p_{01}}{p_{00}} \biggr), \label{theta2} \\
\theta_{12} &=& \log \biggl(\frac{p_{11}p_{00}}{p_{10}p_{01}} \biggr), \label{theta3} 
\end{eqnarray}

and $T(x)$ the vector of sufficient statistics denoted by
\begin{eqnarray}
T(X) = (X_1, X_2, X_1 X_2)^T
\end{eqnarray}
we can rewrite the distribution as an exponential family distribution as follows:
\begin{eqnarray} \label{bivariate_exponential_formulation}
\mathbb{P}(X = x) &=& \exp( \langle T(X) , \theta \rangle - A(\theta) ) 
\end{eqnarray}
where the log partition function $A(\theta)$  is defined such as the probability normalizes to one. It is very easy to check that $A(\theta)= -\log p_{00}$. We can also relate the initial moment parameters $\{p_{ij}\}$ for for $i,j \in \{0, 1\}$, to the canonical parameters as follows

\begin{proposition} \label{prop1}
The moment parameters can be expressed in terms of the canonical parameters as follows;
\begin{eqnarray}
p_{00} = \frac{1}{1 + \exp(\theta_1) + \exp(\theta_2) + \exp(\theta_1 + \theta_2 +
\theta_{12})}, \\
p_{10} = \frac{\exp(\theta_1)}{1 + \exp(\theta_1) + \exp(\theta_2) + \exp(\theta_1 + \theta_2 +
\theta_{12})}, \\
p_{01} = \frac{\exp(\theta_2)}{1 + \exp(\theta_1) + \exp(\theta_2) + \exp(\theta_1 + \theta_2 +
\theta_{12})}, \\
p_{11} = \frac{\exp(\theta_1 + \theta_2 + \theta_{12})}{1 + \exp(\theta_1) + \exp(\theta_2) +
\exp(\theta_1 + \theta_2 + \theta_{12})}.
\end{eqnarray}
\end{proposition}

\begin{proof} 
See \ref{proof1}
\end{proof}

The expression of the distribution in terms of the canonical parameters is particularly useful as it indicates immediately that independence and correlation are equivalent 
as in a Gaussian distribution. We will see that this result generalizes to multivariate binomial in the next subsection \ref{multivariate_binomial}.

\begin{proposition}\label{prop2}
The components of the bivariate Bernoulli random vector $(X_1, X_2)$ are independent if and only if $\theta_{12}$ is zero. 
Like for a normal distribution, independence and correlation are equivalent.
\end{proposition}

\begin{proof} 
See \ref{proof2}
\end{proof}

The equivalence between correlation and independence was already presented in 
\cite{Whittaker_1990} where it was referred to as Proposition~2.4.1. The
importance of $\theta_{12}$  referred to as the cross term or \textit{u-terms}) is discussed and
called \textit{cross-product ratio} between $X_1$ and $X_2$. In \cite{McCullagh_1989}  and \cite{Ma_2010}, this cross product ratio
is also identified but called the log odds.

Intuitively, there are similarities between Bernoulli (their sum version that is the Binomial) and the Gaussian. And like for the multivariate Gaussian, we can prove that the  marginal and the conditional Bernoulli are still binomial  as shown by the proposition \ref{prop3}, making the analogy between Bernoulli (and soon their independent sum version which is the Binomial) and Gaussian even more striking!

\begin{proposition}\label{prop3}
In the bivariate Bernoulli vector whose probability mass function is given by \ref{bivariate} 
\begin{itemize}
\item the marginal distribution of $X_1$  is also a univariate Bernoulli whose probability mass function is 
\begin{equation}
\label{marginalpdf} 
\mathbb{P}(X_1 = x_1) = (p_{10} + p_{11})^{x_1}(p_{00} +p_{01})^{(1-x_1)}.
\end{equation}
\item the conditional distribution of $X_1$ given $X_2$ is also a univariate Bernoulli whose probability mass function is 
\begin{eqnarray}
& & \hspace{-1.6cm} \mathbb{P}(X_1  =  x_1 | X_2 = x_2) =  \biggl(\frac{p_{1 x_2}}{p_{1 x_2} + p_{0 x_2}} \biggr)^{x_1}    \nonumber \\
& & \hspace{-0.5cm} \biggl(\frac{p_{0, x_2}}{p_{1, x_2} + p_{0, x_2}} \biggr)^{1-x_1} \label{conditionalpdf} 
\end{eqnarray}
\end{itemize}
\end{proposition}

\begin{proof} 
See \ref{proof3}
\end{proof}
Before we move on to the generalization, we need to mention a few important facts. 
First, recall that the sum of $n$ independent Bernoulli trials with parameter $p$ is a Binomial with parameter $n$ and $p$. 
And when we talk about independent sum, it should ring a bell! Independent sum should make you immediately think about the Central Limit Theorem. 
This intuition is absolutely correct and shows the connection between the binomial and Gaussian distribution. 
This is the Moivre Laplace theorem stated below

\begin{theorem}\label{theo1}
If the variable $B_n $ follows a binomial distribution with parameters $n$ and $p$ in $]0,1[$, then the variable
$ Z_n = \frac {B_n-np} {\sqrt {np (1-p)}} = \sqrt n \frac{B_n / n - p } {p (1-p)} $ converges in law to a standard normal law $\mathcal{N} (0,1)$. 
Another presentation of this result is to say that, for  $p \in ]0,1[$,
as $n$ grows large, for $k$ in the neighborhood of $n p$ we can approximate the binomial distribution by a normal as follows:
\begin{equation}
\binom{n}{k} p^{k} (1-p)^{n-k} 
\simeq \frac{1}{\sqrt{2 \pi n p (1-p)}}   
e^{-\frac{(k - n p)^2} { 2 n p (1-p)}}
\end{equation}
\end{theorem}

The proof of this theorem is traditionally done with doing a Taylor expansion of the characteristic function. An alternative proof is to use the Sterling formula as well as a Taylor expansion to relate the binomial distribution to the normal one. Historically, de Moivre was the first to establish this theorem in 1733 in the particular case: $p =1 / 2 $. Laplace generalized it in 1812 for any value of $p$ between 0 and 1 and started creating the ground for the central limit theorem that extended this result far beyond. Later on, many more mathematicians generalized and extended this result like Cauchy, Bessel, Poisson but also von Mises, Pólya, Lindeberg, Lévy, Cramér as well as Chebyshev, Markov and Lyapunov.

Second, if we take an infinite sum of Bernoulli, this is a discrete distribution that is the asymptotic limit of the binomial distribution. 
This is also a distribution that is part of the exponential family and is given by the Poisson distribution.

\begin{proposition}\label{prop4}
For a large number $n$ of independent Bernoulli trials with probability $p$ such that $\lim\limits_{n\to\infty} np  = \lambda$, then the corresponding binomial distribution with parameter $n$ and $p$ converges in distribution to the Poisson distribution
\end{proposition}
\begin{proof} 
See \ref{proof4}
\end{proof}

The two previous results show that binomials, Poisson and Gaussian distributions that are part of the exponential family are closely connected and represent the discrete and continuous version of very similar concepts, namely independent and identically distributed increments. 

\subsection{Maximum entropy}
It is also interesting to relate these distributions to maximum Shannon entropy. Let a function : $\Phi : \Xi \to \mathbb{R}^d$, where $\Xi$ is the space of the random variable $X$ and a vector $\alpha \in \mathbb{R}^d$. It is well known that the maximum entropy distribution whose constraint is given by $\mathbb{E}_P \left[ \Phi(X) \right] = \alpha$ is a distribution of the exponential family given by the following theorem

\begin{theorem}\label{theo2}
The distribution that maximizes the Shannon entropy : $- \int p(x) \log p(x) d\mu(x)$
subject to  the constraint $\mathbb{E}_P \left[ \phi(X) \right] = \alpha$ and the obvious probability constraints $ \int p(x) d\mu(x) = 1$, $p(x) \geq 0$,
is the unique distribution that is part of the exponential family and given by
\begin{equation}
p_{\theta} = \exp( < \theta, \Phi(x) > - A(\theta) ),
\end{equation}
with 
$$ 
A(\theta ) = \log \int \exp( < \theta, \Phi(x) > ) d\mu(x) 
$$
\end{theorem}
\begin{proof} 
See \ref{proof5}
\end{proof}

\begin{rem}
The theorem \ref{theo2} works also for discrete distributions. It says that the discrete distribution that maximizes the Shannon entropy  $- \sum p(x) \log p(x)$ subject to the constraint
$\mathbb{E}_P \left[ \phi(X) \right] = \alpha$ and the obvious probability constraints $ \sum p(x) = 1$, $p(x) \geq 0$,
is the unique distribution that is part of the exponential family and given by
\begin{equation}
p_{\theta} = \exp( < \theta, \Phi(x) > - A(\theta) ),
\end{equation}
with 
$$ 
A(\theta ) = \log \int \sum ( < \theta, \Phi(x) > ) 
$$
\end{rem}

The theorem \ref{theo2} implies in particular that the continuous distribution that maximizes entropy with given mean and variance (or equivalently first and second moments) is an exponential family of the form 
$$
\exp( \theta_1 x + \theta_2 x^2 - A(\theta) )
$$
where the log partition function $A(\theta)$ is defined to ensure the probability distribution sums to one. This distribution is indeed a normal distribution as it is the exponential of a quadratic form. This is precisely the continuous distribution used in the CMA ES algorithm. Taking a distribution that maximizes the entropy means that we take a distribution that has the less information prior. Or said differently, this is the distribution with the minimal prior structural constraint. If nothing is known, it should therefore be preferred.

Ideally, for our CMA ES discrete adaptation, we would like to find the discrete distribution equivalent of the normal. if we want the discrete distribution with independent increment, we should turn to binomial distributions. Binomials have the other advantage to converge to the normal distribution whenever the discrete parameter converges to a continuous one. Binomials are also distributions that are part of the exponential family.  But we are facing various problems. To keep the discussion simple, let us first look at a single parameter that can take as values all the integer between $0$ to $n$

First of all, we face the issue of controlling the first two moments of our distribution or equivalent to be able to control the mean denoted by $\mu$ and the variance denoted by $\sigma^2$. Binomial distributions do not have two parameters like normals to be able to adapt to first and second moments constraints as easily as normals.  Indeed for our given parameter $n$  that is the number of discrete state of our parameter to optimize in the discrete CMA ES, we are only left with a single parameter $p$ for our binomial distribution $\mathcal{B}(n,p)$ to accommodate for the constraints. The expectation is given by $n p$ while the variance is given by $n p (1-p)$. If we would like to have a discrete distribution that progressively peaks to the minimum, we would like to be able to force the variance to converge to $0$. This will fix the variance to $\sigma^2 = np (1-p)$. We can easily solve this quadratic equation $p^2 - p + \sigma^2 / n = 0$ and use the minimal solution given by 
$$
p =  \frac{1- \sqrt{ 1 - 4 \sigma^2 / n}}{2} 
$$
provided that $\sigma^2 \leq n / 4$. As $\sigma$ will tend to zero, the parameter $p$ will tend to zero. In order to accommodate for the mean constraint, we need a work-around. We see that the discrete parameter is we do not do anything will converge to $0$ as $p$ will converge to $0$. A solution that is simple is to assume that our discrete parameter is distributed according to 
$$
(\mu + (\mathcal{B}(n,p) - n p) ) \mod n
$$
where $a \mod n$ is a modulo $n$. It is the remainder of the Euclidean division of a by $n$. This method will ensure that we sample all possible $0$ to $n$ possible value with a mean that is equal to $\mu$ and a variance controlled by the parameter $p$.

Secondly, we would like to use a discrete distribution that maximizes the entropy. This is the case for the continuous version of CMA-ES with the normal distribution. However, for discrete distribution, this maximum entropy condition is not as easy. It is well known that the maximum entropy discrete distribution with a given mean is not the binomial distribution but rather the distribution given by 
$$
p(X=i) = c \,\, \rho^i
$$
where $c = 1 / {\sum\limits_{i=0}^n \rho^i} = \frac{1-\rho}{1 - \rho^{n+1}} $ and where $\rho$ is determined such as $\sum\limits_{i=0}^n c \,\, i \,\, \rho^i = \mu$ which leads to the implicit equation for $\rho$: 
$$
(1 + \mu) \rho + (\mu- (n+1)) \rho^{n+1} + (n -\mu) \rho^{n+2}= \mu,
$$
using the well known geometric identities: $\sum\limits_{i=0}^n \rho^i =  \frac{1-\rho^{n+1}}{1 - \rho}$ and $\sum\limits_{i=0}^n i \rho^i =  \frac{  \rho  \frac{1-\rho^{n+1}}{1 - \rho} - (n+1) \rho^{n+1}} {1 - \rho}_{_{_{}}}$. The distribution is sometimes referred to as the truncated geometric distribution. This is not our desirable binomial distribution. Obviously, we can rule out this truncated geometric distribution as its probability mass function does not make sense for our parameter. The probability mass function is decreasing which is not a desirable feature. Rather, we would like a bell shape, which is the case for our binomial distribution. The tricky question is how to relate our binomial distribution to a maximum entropy principle as this is the case for the normal. 

We can first remark that the binomial distribution is not too far away from a geometric distribution when the number of trials $n$ tends to infinity at least for some terms. Indeed, the probability mass function is given by $\binom{n}{k} p^k (1-p)^{n-k}$. And using the Sterling formula, we can see that for $n$ large, we can approximate factorial $n$ as follows $n\,! \sim \sqrt{2\pi n}\,\left( \frac{n}{e}\right)^n$, which leads an asymptotic term similar to the geometric distribution. This gives some hope that there should be a way to relate our binomial distribution to a maximum entropy principle. And the trick here is to reduce the space of possible distributions. It instead of looking at the entire space of distribution, we reduce the space of possible distributions to any Poisson binomial distributions (also referred to in the statistics literature as the generalized binomial distribution), we could find a solution. The latter distribution named after the famous French mathematician \textit{Siméon Denis Poisson} is the discrete probability distribution of a sum of independent Bernoulli trials that are not necessarily identically distributed. And nicely, restricting the space of possible distribution to any Poisson binomial distributions, theorem \ref{theo3} proves that the binomial distribution is the distribution that maximizes the entropy for a given mean.

\begin{theorem}\label{theo3}
Among all Poisson binomial distributions with $n$ trials, the distribution that maximizes the Shannon entropy : $- \sum p(x) \log p(x)$
subject to  the constraint $\sum x p(x) = \mu$ and the obvious probability constraints is the binomial distribution $\mathcal{B}(n,p)$
such that $n p = \mu$
\end{theorem}
\begin{proof} 
See \ref{proof6}
\end{proof}

\subsection{Multivariate and Correlated Binomials} \label{multivariate_binomial}
Equipped with the intuition of the first section, we can see the profound connection between multivariate normal and multivariate binomial. 
We will define our multivariate binomial as the sum for $n$ independent trials of multivariate Bernoulli defined as before.

Let $X= (X_1, \ldots, X_k)$ be a k-dimensional random vector of 
possibly correlated binomial random variables that may have different parameters $n_i$ and $p_i$ and
let $x = (x_1,\ldots ,x_k)$ be a realization of $X$. The joint probability is given naturally by
\small
\begin{align} \label{mvbinomial:proba}
& \quad \mathbb{P}(X_1 = x_1, X_2 = x_2, \ldots, X_K = x_K)  \nonumber
\\
& = p(0, 0, \ldots,0)^{\prod_{j=1}^K (1-x_j)} \times p(1, 0, \ldots, 0)^{x_1 \prod_{j=2}^K(1-x_j)}  \nonumber
\\
&  \quad  \times  p(0, 1, \ldots, 0)^{(1-x_1)x_2 \prod_{j=3}^K(1-x_j)} \,\, \times \ldots \times \nonumber
\\
& \quad \times p(1, 1, \ldots, 1) ^{\prod_{j=1}^K x_j},
\end{align}
\normalsize

Like for the simple case of section \ref{prim:sec}, we can re-write this joint probability in the exponential form. 
Let us give some notations. 

Let  $T(X)$ be the vector $( X_1, ..., X_k,$ $ X_1 X_2, \ldots, X_1 \ldots X_k)^T$ of size $2^k-1$ whose elements represents all the possible $1$ to $k$ selection of $X_1, \ldots, X_k$. 
These  $1$ to $k$ selections of $X_1, \ldots, X_k$ are all the possible monomial polynomials of $X_1, \ldots, X_k$ of degree $1$ to $k$. By monomial, we mean that we can take only distinct power of  $X_1, \ldots, X_k$ with all of them having an exponent equal to 0 or 1. We also denote by $(i_1, \ldots, i_l)$ an ordered set of $1 \leq l \leq k$ elements of the integers from $1$ to $k$ and by $\Upsilon_{\{1,\ldots,k\}}$ the set of all the order sets $(i_1, \ldots, i_l)$ with $1 \leq l \leq k$ elements elements.  $\Upsilon_{\{1,\ldots,k\}}$ is also the sets of all possible non empty sets with integer elements in $\{1, \ldots, k\}$. Similarly,  $\Upsilon_{\{ i_1,\ldots, i_l \}}$ is the sets of all possible non empty set with elements in $\{ i_1,\ldots, i_l \}$. Finally, $\Upsilon_{\{ i_1,\ldots, i_l \}}^{even}$ (respectively  $\Upsilon_{\{ i_1,\ldots, i_l \}}^{odd}$) is the subset of $\Upsilon_{\{ i_1,\ldots, i_l \}}$ for set whose cardinality is even (respectively odd).

We are now able to provide the following proposition that gives the exponential form of the multi variate binomial mass probability function:

\begin{proposition}[Exponential form]\label{mvbinomial:parameters}
The multivariate Bernoulli model has a probability mass function of the exponential form given by
\begin{eqnarray}\label{mvbinomial:exponential_form}
 \mathbb{P}(X) = \exp( <\theta, T(X) > - A(\theta))
\end{eqnarray}
where the sufficient statistic $T(X) $ is $T(X)=(X_1, ..., X_k, X_1 X_2, \ldots, X_1 \ldots X_k)$, the log partition function $A(\theta)$ is $A(\theta)= -\log p(0, 0, \ldots,0)$ and the coefficients $\theta$ are given by:
\begin{eqnarray}
\theta_{i_1,\ldots, i_l} &=& \log \frac{ Even }{ Odd}, \label{mvbinomial:exponential_form2} \\
\text{with } Even \hspace{-0.2cm}&=&  \hspace{-0.7cm} \prod_{\substack{\{j_1, .., j_m\} \in \Upsilon_{\{ i_1,.., i_l \}}^{even}}} \hspace{-0.6cm} p\left(\substack{ 1 \text{ for  all } i_1, \ldots,i_l \\  \text{ but }  j_1, \ldots, j_m \text{ with }0 \\ \text{ rest with }0}\right) \label{mvbinomial:exponential_form3} \\
\text{and  } Odd \hspace{-0.2cm}&=&  \hspace{-0.7cm} \prod_{\substack{\{j_1, .., j_m\} \in \Upsilon_{\{ i_1,.., i_l \}}^{odd  }}} \hspace{-0.6cm} p\left(\substack{ 1 \text{ for  all } i_1, \ldots,i_l \\   \text{ but } j_1, \ldots, j_m \text{ with }0 \\ \text{ rest with }0}\right) \label{mvbinomial:exponential_form4}
\end{eqnarray}

Similarly, we can compute the regular probabilities from the canonical parameters as follows:
\begin{eqnarray}
p(\substack{ 1 \text{ for } i_1, \ldots, i_l \\ \text{ rest with }0}) = \frac{\exp( S^{i_1,\ldots, i_l}) }{D}. \label{mvbinomial:proba1}
\end{eqnarray}
where $S^{i_1,\ldots, i_l} $ is the sum of all the theta parameters indexed by any non empty selection within $\{ i_1,\ldots, i_l \}$:
\begin{eqnarray}
S^{i_1,\ldots, i_l} =\sum\limits_{\{ i_1, \ldots, i_m\} \in \Upsilon_{\{ i_1,\ldots, i_l \}}} \theta_{ i_1, \ldots, i_m} \label{mvbinomial:proba2}
\end{eqnarray}
with the convention for the empty set, $S=1$
and $D$ is the normalizing constant such that all the probabilities sum to 1:
\begin{eqnarray}
D =\sum_ {\substack{l=0, .. , k \\ 1 \leq i_1  \leq \ldots  \leq i_l }} \exp( S^{i_1, \ldots, i_l})  \label{mvbinomial:proba3}
\end{eqnarray}
with the convention for $l=0$ that $\exp( S^{i_1, \ldots, i_l})=1$.
\end{proposition}

\begin{proof} 
See \ref{proof7}
\end{proof}

Last but not least, we can extend the result already found for the simple two dimension Bernoulli variable to the general multi dimensional Bernoulli concerning independence and correlation. Recall that one of the important statistical properties for the multivariate Gaussian distribution is the equivalence of independence and no correlation. This is a remarkable properties of the Gaussian (although more could be said about independent and Gaussian as explained for instance in \cite{Benhamou_remarkable}).

The independence of a random vector is determined by the separability of coordinates in its probability mass function. If we use the natural (or moment) parameter form of the probability mass function, this is not obvious. However, using the exponential form, the result is almost trivial and is given by the following proposition
\begin{proposition}[(Independence of Bernoulli outcomes)]\label{mvbinomial:independence}
The multivariate Bernoulli variable $X= (X_1, \ldots, X_k)$ is independent element-wise if and only if
\begin{equation}\label{mvbinomial:independence2}
\theta_{i_1, \ldots, i_l} = 0 \qquad \forall 1\leq i_1< \ldots< i_l \leq k, l \geq 2.
\end{equation}
\end{proposition}

\begin{proof} 
See \ref{proof8}
\end{proof}

\begin{rem}
The condition of equivalence between independence and no correlation can also be rewritten as
\begin{equation}
S^{i_1, \ldots, i_l} = \sum_{k=1}^l \theta_{i_k} \qquad\forall l \geq2.
\end{equation}
\end{rem}

\begin{rem}
A general multi variate binomial model implies $2^n-1$ parameters, which is way to many when $n$ is large. A simpler model is to impose that only the probabilities involving one state $X_i$ or two states $X_i, X_j$ are non zero. This is in fact the Ising model.
\end{rem}

\section{Algorithm}\label{sec:algorithm}
\subsubsection{CMA-ES estimation}
Another radically difference approach is to minimize some cost function depending on the Kalman filter parameters. As opposed to the maximum likelihood approach that tries to find the best suitable distribution that fits the data, this approach can somehow factor in some noise and directly target a cost function that is our final result. Because our model is an approximation of the reality, this noise introduction may leads to a better overall cost function but a worse distribution in terms of fit to the data. 

Let us first introduce the CMA-ES algorithm. Its name stands for covariance matrix adaptation evolution strategy. As it points out, it is an evolution strategy optimization method, meaning that it is a derivative free method that can accomodate non convex optimization problem. The terminology covariance matrix alludes to the fact that the exploration of new points is based on a multinomial distribution whose covariance matrix is progressively determined at each iteration. Hence the covariance matrix adapts in a sense to the sampling space, contracts in dimension that are useless and expands in dimension where natural gradient is steep. This algorithm has led to a large number of papers and articles and we refer to \cite{Hansen_2018}, \cite{Ollivier_2017}, \cite{Auger_2016}, \cite{Auger_2015}, \cite{Hansen_2014}, \cite{Auger_2012}, \cite{Hansen_2011}, \cite{Auger_2009}, \cite{Igel_2007}, \cite{Auger_2004} to cite a few of the numerous articles around CMA-ES. We also refer the reader to the excellent wikipedia page \cite{wiki:CMAES}.

In order to adapt CMA ES to discrete variables, we change in the algorithm the generation so Gaussian variables into the ones of multi variate binomials as follows:
$$
x_i \sim m + \mathcal{B}( \sigma^2 C ) \mod \text{dim}
$$
The corresponding algorithm is given in \ref{DiscreteCMAES}.


\section{Conclusion}
\label{sec:conclusion}
In this paper, we showed that using multi-variate correlated binomial distribution, we can derive an efficient adaptation of CMA-ES for discrete variable optimization problem using correlated binomials. We have proved that correlated binomials share some similarities with normal distribution in terms of independence and correlation equivalence as well as rich information for correlation structure. In order to avoid too many parameters, we impose that only single state and bi-state probabilities are not null. In the future, we hope to develop additional variations around this CMA-ES version for the combination of discrete and continuous variables mixing potentially multivariate binomial and normal distributions.
\bibliographystyle{plain}
\bibliography{mybib}

\clearpage

\onecolumn

\appendix
\section{Proofs}

\subsection{Proof of proposition \ref{prop1}}\label{proof1}
\begin{proof} We can trivially infer all the moment parameters from equations \ref{theta1}, \ref{theta2} and \ref{theta3}.
\end{proof}

\subsection{Proof of proposition \ref{prop2}}\label{proof2}
\begin{proof} 
The exponential family formulation of the bivariate Bernoulli distribution shows that a necessary and sufficient condition for the distribution to seperable 
into two components with each only depending on $x_1$ and $x_2$ respectively is that $\theta_{12}=0$. This proves the first assertion of proposition \ref{prop2}.

Proving equivalence between correlation and independence is the same as proving equivalence between covariance and independence. 
The covariance between $X_1$ and $X_2$ is easy to calculate and given by 
\begin{eqnarray}
\operatorname{Cov}(X_1, X_2) &= & \mathbb{E}\left[ X_1 X_2 \right] - \mathbb{E}\left[ X_1 \right] \mathbb{E}\left[ X_2 \right]  \\ 
& =& K e^{\theta_0 + \theta_1 + \theta_{12}} - K  (e^{\theta_0} + e^{\theta_0 + \theta_1 + \theta_{12}}) )   K (  e^{\theta_1} + e^{\theta_0 + \theta_1 + \theta_{12}} ) \\
&= & K  e^{\theta_0 + \theta_1 + \theta_{12}} (1 - K e^{\theta_0} - K e^{\theta_1} - K e^{\theta_0 + \theta_1 + \theta_{12}}) - K^2 e^{\theta_0 + \theta_1} \\ \label{sum_to_one}
& = & K ^2  e^{\theta_0 + \theta_1} (e^{\theta_{12}} - 1 ) 
\end{eqnarray}
where in equation \ref{sum_to_one}, we have used that the four probabilities sum to one.
Hence, the correlation or the covariance is null for non trivial probabilities if and only if $\theta_{12} = 0$, which is equivalent to the independence.
\end{proof}

\subsection{Proof of proposition \ref{prop3}}\label{proof3}
\begin{proof} 
For the coordinate $X_1$, it is trivial to see that 
\begin{eqnarray*}
\mathbb{P}(X_1 = 1)  &=& P(X_1 = 1, X_2 = 0) + P(X_1 = 1, X_2 = 1) = p_{10} + p_{11}, \\
\mathbb{P}(X_1 = 0) &= & p_{00} + p_{01}, \\
\mathbb{P}(X_1 = 1) &+& P(X_1 = 0) = 1.
\end{eqnarray*}
which shows that $X_1$ follows the univariate Bernoulli distribution with density given by equation \eqref{marginalpdf}.
Likewise, it is trivial to see that
\begin{eqnarray*}
\mathbb{P}(X_1 = 0 | X_2 = 0) &=& \frac{P(X_1 = 0, X_2 = 0)}{P(X_2 = 0)} =  \frac{p_{00}}{p_{00} + p_{10}}, \\
\mathbb{P}(X_1 = 1 | X_2 = 0) &=& \frac{p_{10}}{p_{00} + p_{10}}, \\
\mathbb{P}(X_1 = 1 | X_2 = 0) &+ &P(X_1 = 0 | X_2 = 0) = 1. 
\end{eqnarray*}
Similar results apply for the condition $X_2 = 1$, which shows the second result and concludes the proof.
\end{proof}

\subsection{Proof of proposition \ref{prop4}}\label{proof4}
\begin{proof} 
Let us write the limit for the binomial distribution when number of trials $n\to\infty$, and probability of success in trial $p \to 0$ but $np \to \lambda$ remains finite. We have for a given $k$

\begin{eqnarray}
&& \lim\limits_{n\to\infty}\displaystyle\binom{n}{k}\left(\dfrac{\lambda}{n}\right)^k\left(1-\dfrac{\lambda}{n}\right)^{n-k}\\ &&
=\dfrac{\lambda^k}{k!}\cdot\lim\limits_{n\to\infty}\left[1\left(1-\dfrac1n\right)\left(1-\dfrac2n\right)\ldots\left(1-\dfrac{k-1}{n}\right)\right]\cdot\lim\limits_{n\to\infty}\left(1-\dfrac{\lambda}{n}\right)^n\cdot\dfrac{1}{\lim\limits_{n\to\infty}\left(1-\dfrac{\lambda}{n}\right)^k} \\ & &
=\lim\limits_{n\to\infty}\left(1-\dfrac{\lambda}{n}\right)^n\cdot\dfrac{\lambda^k}{k!}\cdot\lim\limits_{n\to\infty}\left[1\left(1-\dfrac1n\right)\left(1-\dfrac2n\right)\ldots\left(1-\dfrac{k-1}{n}\right)\right]\cdot\dfrac{1}{\lim\limits_{n\to\infty}\left(1-\dfrac{\lambda}{n}\right)^k} \\ & &
=\dfrac{e^{-\lambda}\lambda^k}{k!}.
\end{eqnarray}
which proves that the binomial converges to the Poisson distribution.
\end{proof}

\subsection{Proof of theorem \ref{theo2}}\label{proof5}
\begin{proof}
We follow the proof of Theorem 11.1.1 of \cite{Cover_2006}. If we write the Lagrangian $\mathcal{L}(p,\theta, \theta_0, \lambda)$  for the problem 
\begin{eqnarray*}
& \text{maximize} \quad - \int p(x) \log p(x) d\mu(x) \\ 
& \text{subject to} \quad  \mathbb{E}_P \left[ \phi(X) \right] = \alpha    \quad \text{and } \quad \int p(x) d\mu(x) = 1  \quad \text{and } \quad p(x) \geq 0 
\end{eqnarray*}
where the Lagrange multipliers $(\theta, \theta_0, \lambda)$ are for the three constraints, we have
\begin{eqnarray*}
\mathcal{L}(p,\theta, \theta_0, \lambda) = \int  p(x) \log p(x) d\mu(x)  + \sum_{i=1}^d \theta_i \left( \alpha_i -  \int  p(x) \phi_i(x) d\mu(x)  \right) \\
+ \theta_0 \left( \int p(x) d\mu(x) - 1 \right) - \int   \lambda(x) p(x)    d\mu(x)
\end{eqnarray*}
and noticing that the function to optimize is convex and satisfies the Slater’s constraint, we can use Lagrange duality to characterize the solution as the solution of the critical point given by
\begin{eqnarray}
1 +\log p(x) - < \theta, \phi(x)> + \theta_0 - \lambda(x) = 0
\end{eqnarray}
or equivalently, 
\begin{eqnarray}
p(x) = \exp(  < \theta, \phi(x)> -1 - \theta_0 + \lambda(x) )
\end{eqnarray}
As this solution always satisfies the condition $p(x) > 0$, we have necessarily that the Lagragian multiplier related to the constraint  $p(x) > 0$ should be null: $\lambda(x) = 0$.  The solution should be a probability distribution, which implies that 
$$
\int p(x)  d\mu(x) = 1
$$
which imposes that
\begin{eqnarray*}
\int  \exp(  -1 - \theta_0  ) d\mu(x) = \int \exp(  < \theta, \phi(x)> ) d\mu(x)
\end{eqnarray*}

or equivalently, writing in the exponential form $\theta_0 -1 = A(\theta) =  \log \int \exp(  < \theta, \phi(x)> )d\mu(x)$, 
we have that $p$ satisfies
\begin{eqnarray}
p_{\theta}(x) = \exp(  < \theta, \phi(x) - A(\theta) >)
\end{eqnarray}
which shows that the distribution is part of the exponential family.  

To prove its uniqueness, we use the fact that the Shannon entropy is related to the Kullback Leibler divergence $D_{kl}$ as follows:
\begin{eqnarray}
H(P) =  - \int p(x) \log p(x) d\mu(x)  = - D_{kl}( P \| P_{\theta_0}) - H(P_{\theta_0})
\end{eqnarray}
which concludes the proof as the  Kullback Leibler divergence $ D_{kl}( P \| P_{\theta_0}) > 0$ unless $P= P_{\theta_0}$
\end{proof}

\subsection{Proof of theorem \ref{theo3}}\label{proof6}
\begin{proof}
We will prove a stronger result that the entropy $H(p_1, \ldots, p_n)$ of a generalized binomial distribution with parameters $n, p_1, \ldots, p_n$ is Schur concave (see \cite{Pecaric_1992} for a definition and some properties). A straight consequence of Schur concavity is that the function is maximum for the constant function as follows:
\begin{eqnarray}
H(p_1, \ldots, p_n) \leq H(\bar{p}, \ldots, \bar{p})
\end{eqnarray}
with $\bar{p} = \frac{\sum_{i=1}^n p_i}{n}$. This will prove that the regular binomial distribution satisfies the maximum entropy principle.

Our proof of the Schur concavity uses the same trick as in \cite{Harremoes_2001}, namely the usage of elementary symmetric functions. 
Let us denote by $(X_i)_{i=1,\ldots,n}$ the independent Bernoulli variables with parameter $p_i$ and their sum $S_n= \sum_{i=1}^n X_i$ the variable for the canonical Poisson binomial variable. Its probability mass function writes as
\begin{equation}
\pi_k^n \widehat{=} \Pr(S_n=k) = \sum\limits_{A\in F_k} \prod\limits_{i\in A} p_i \prod\limits_{j\in A^c} (1-p_j) 
\end{equation}

where $F_k$ is the set of all subsets of $k$ integers selected from $\{1,2,3,...,n\}$. The entropy $H(p_1, \ldots, p_n)$ is permutation symmetric, hence to prove Schur concavity, it suffices to show that the cross term $(p_1 - p_2) (\frac{\partial H}{p_1} - \frac{\partial H}{p_2})$ is negative. Let us compute

\begin{equation}\label{proof6:eq1}
\frac{\partial H}{p_1} - \frac{\partial H}{p_2} = - \sum_{k=0}^n (1+ \log \pi_k^n) (\frac{\partial \pi_k^n}{p_1} - \frac{\partial \pi_k^n}{p_2})
\end{equation}

We can notice that for $k \geq 2$ and $k \leq n-2$
\begin{equation}\label{proof6:eq2}
\pi_k^n = p_1 p_2 \pi_{k-2}^{n-2} + (p_1 (1-p_2) + (1-p_1) p_2)  \pi_{k-1}^{n-2} +  (1-p_1) (1-p_2)   \pi_{k}^{n-2}
\end{equation}
where $ \pi_{j}^{n-2} = \pi_{j}^{n-2}( p_3, \ldots, p_n)$. The equation \eqref{proof6:eq2} can be extended for $k=0, 1$ or $k=n-1, n$ with the convention that $\pi_{j}^{n}= 0$ for $j<0$ and $\pi_{j}^{n}= 0$ for $j>n$. Hence equation \eqref{proof6:eq2} is valid for any $k$. Deriving equation \eqref{proof6:eq2} with respect to $p_i$  leads to

\begin{equation}\label{proof6:eq3}
\frac{\partial \pi_k^n}{p_1} - \frac{\partial \pi_k^n}{p_2} = - (p_1-p_2) (  \pi_{k-2}^{n-2} - 2  \pi_{k-1}^{n-2} +  \pi_{k}^{n-2} )
\end{equation}

Combining equations \eqref{proof6:eq1} and \eqref{proof6:eq3}, we have
\begin{eqnarray}\label{proof6:eq4}
(p_1 - p_2) (\frac{\partial H}{p_1} - \frac{\partial H}{p_2}) &= &(p_1 - p_2)^2  \sum_{k=0}^n (1+ \log \pi_k^n) (\pi_{k-2}^{n-2} - 2  \pi_{k-1}^{n-2} +  \pi_{k}^{n-2}) \nonumber \\
& = & (p_1 - p_2)^2  \sum_{k=0}^{n} (\log \pi_k^n) (\pi_{k-2}^{n-2} - 2  \pi_{k-1}^{n-2} +  \pi_{k}^{n-2}) \nonumber \\
& = & (p_1 - p_2)^2  \sum_{k=0}^{n-2} \pi_{k}^{n-2} \log  \frac{ \pi_{k+2}^n \pi_{k}^n }{ (\pi_{k+1}^n)^2} \label{proof6:eq5}
\end{eqnarray}

Recall a result that is allegedly attributed to Newton about elementary symmetric functions. Denote the product
\begin{eqnarray}
(x+a_1)(x+a_2)\ldots (x+a_n) &=& x^n+ c_1 x^{n-1}+ c_2 x^{n-2}+\ldots +c_n
\end{eqnarray}
with $c_k$ the k\textsuperscript{th} elementary function of the a's. We have
\begin{equation}\label{proof6:newton}
c_{k-1} c_{k+1} < c_k^2
\end{equation}
unless all a are equal (see for instance \cite{Hardy_1988} theorem 51 page 52 section 2.22). Let us take the function
\begin{eqnarray}
\prod_{i=1}^n  (1-p_i) (x + \frac{p_1}{1-p_1})(x + \frac{p_2}{1-p_2})\ldots (x+ \frac{p_n}{1-p_n}) &=& \prod_{i=1}^n  (1-p_i) (x^n+ c_1 x^{n-1}+ c_2 x^{n-2}+\ldots +c_n) \\
&=& x^n + \pi_1^n x^{n-1}+ \pi_2^n x^{n-2}+ \ldots + \pi_n^n
\end{eqnarray}
we have therefore that 
\begin{eqnarray}\label{proof6:relationship}
\pi_k^n = \prod_{i=1}^n  (1-p_i) c_{k}
\end{eqnarray}
Combining equations \eqref{proof6:newton} and \eqref{proof6:relationship}, we proved that
\begin{eqnarray}
\pi_{k+2}^n \pi_{k}^n < (\pi_{k+1}^n)^2 
\end{eqnarray}
which concludes the proof
\end{proof}

\subsection{Proof of proposition\ref{mvbinomial:parameters}}\label{proof7}
\begin{proof}
Comparing the equations \eqref{mvbinomial:proba} and \eqref{mvbinomial:exponential_form}, and using the provided sufficient statistics given in  proposition\ref{mvbinomial:parameters}, we see by identification that the parameters $\theta$ should be given by the equation \eqref{mvbinomial:exponential_form2} with the terms with a plus sign given by equation \eqref{mvbinomial:exponential_form3} and the terms with a negative sign given by equation \eqref{mvbinomial:exponential_form4}

Similarly, if we take equation \eqref{mvbinomial:proba1},  \eqref{mvbinomial:proba2} and \eqref{mvbinomial:proba3}, we can notice that the probabilities given sum to one, that $D= 1 / p_{0, \ldots, 0}$ and that from the expression giving the theta's \eqref{mvbinomial:exponential_form2}, we back out the probabilities. This concludes the proof.
\end{proof}

\subsection{Proof of proposition\ref{mvbinomial:independence}}\label{proof8}
\begin{proof}
The proof of proposition \ref{mvbinomial:independence} is immediate using the exponential form as the independence is equivalent to the separability which is equivalent to equation \eqref{mvbinomial:independence2}
\end{proof}

\section{Pseudo code}
\begin{algorithm}[H]
\caption{CMA ES algorithm} \label{DiscreteCMAES}
	\begin{algorithmic} 
	\State \textbf{Set} $\lambda$								\Comment{number of samples /  iteration}
	\State Initialize $m, \sigma, C=\mathbf{I}_n, p_\sigma=0$, $p_c=0$  \Comment{initialize state variables}
	\\
	\While{(not terminated)}  								
     		\For{$i = 1 \text{ to } \lambda$} 					\Comment{samples $\lambda$ new solutions and evaluate them}
        		\State $x_i \sim m + \mathcal{B}( \sigma^2 C ) \mod \text{dim}$ 	\Comment{samples multivariate correlated Bernoulli}
	        	\State $f_i =f(x_i)$								\Comment{evaluates}
		\EndFor
		\\
	      	\State $x_{1...\lambda} = $  $x_{s(1)...s(\lambda)}$ with $s(i)$ = argsort($f_{1...\lambda}$, $i$)  \Comment{reorders samples}
     		\State $m' = m$  									\Comment{stores current value of $m$}
		\\
	     	\State $m = $ update mean$(x_1, ... ,$ $x_\lambda)$  \Comment{udpates mean to better solutions}
     		\State $p_\sigma = $ update ps$(p_\sigma,$ $\sigma^{-1} C^{-1/2} (m - m'))$  \Comment{updates isotropic evolution path}
	     	\State $p_c = $ update pc$(p_c,$ $\sigma^{-1}(m - m'),$ $||p_\sigma||)$  \Comment{updates anisotropic evolution path}
	     	\State $C = $ update C$(C,$ $p_c,$ ${(x_1 - m')}/{\sigma},... ,$ ${(x_\lambda - m')}/{\sigma})$  \Comment{updates covariance matrix}
	     	\State $\sigma = $ update sigma$(\sigma,$ $||p_\sigma||)$  \Comment{updates step-size using isotropic path length}
		\\
		\State not terminated $=$ iteration $\leq$ iteration max and $||m-m'|| \geq \varepsilon$ \Comment{stop condition}
	\EndWhile
	\\ \\
	\Return $m$ or $x_1$			\Comment{returns solution}
	\end{algorithmic}
\end{algorithm}

\end{document}